\documentclass{article}

\usepackage[nonatbib, preprint, logo]{cig}

\usepackage[utf8]{inputenc} %
\usepackage[T1]{fontenc}    %
\usepackage{hyperref}       %
\usepackage{arydshln}
\usepackage[export]{adjustbox}
\usepackage{tikz}
\usepackage{cite}
\usepackage{graphicx}
\usepackage{amsthm} 
\usepackage{amsmath} 
\usepackage{amssymb}
\usepackage{bm}
\usepackage{algorithm}
\usepackage{algorithmic}
\usepackage{booktabs}
\usepackage{xcolor}
\usepackage{xspace}
\usepackage{soul}
\usepackage{cite}
\usepackage{orcidlink}
\usepackage{inconsolata}   %

\usepackage{booktabs}  %
\usepackage{multirow}
\usepackage[table]{xcolor}

\usepackage{hyperref}
\hypersetup{
  colorlinks=false,
  pdfborder={0 0 0}     %
}

\title{\methodabbrev: One-Step Generative Reconstruction}

\vspace{5em}
\author{%
\normalsize Xiao~Shi$^{\dagger,1}$ \quad Edward~P.~Chandler$^{\dagger,1}$ \quad  Chicago~Y.~Park$^1$ \\[0.3em] Shirin~Shoushtari$^2$ \quad Ulugbek~S.~Kamilov$^1$\\[0.7em]
\small \textnormal{$^1$University of Wisconsin-Madison} \quad \textnormal{$^2$Washington University in St.\ Louis}\\[0.5em]
\footnotesize \texttt{\{xiao.shi, epchandler\}@wisc.edu}\\[0.5em]
$^\dagger$~Equal Contribution
}

\def\d{{\mathrm{\, d}}}

\def\min{\mathop{\mathsf{min}}}

\def\E{\mathbb{E}}

\def\R{\mathbb{R}}

\def\zerobm{{\bm{0}}}

\def\nbm{{\bm{n}}}

\def\ubm{{\bm{u}}}
\def\vbm{{\bm{v}}}
\def\wbm{{\bm{w}}}
\def\xbm{{\bm{x}}}
\def\ybm{{\bm{y}}}
\def\zbm{{\bm{z}}}

\def\mubm{{\bm{\mu}}}

\def\psibm{{\bm{\psi}}}
\def\epsilonbm{{\bm{\epsilon}}}

\def\phibm{{\bm{\phi}}}

\def\thetabm{{\bm{\theta }}}

\def\Abm{{\bm{A}}}

\def\Ibm{{\bm{I}}}

\def\Pbm{{\bm{P}}}

\def\Vbm{{\bm{V}}}

\def\Lcal{{\mathcal{L}}}
\def\Mcal{{\mathcal{M}}}
\def\Ncal{{\mathcal{N}}}

\def\Ucal{{\mathcal{U}}}

\newcommand{\hlgreen}[1]{{\sethlcolor{lightgreen}\hl{#1}}}
\newcommand{\hlblue}[1]{{\sethlcolor{lightblue}\hl{#1}}}

\newcommand{\vecfield}{\vbm}
\newcommand{\meanvecfield}{\ubm}
\newcommand{\startDist}{q}
\newcommand{\imageDist}{p}
\newcommand\flow{{\phibm}}
\newcommand{\conditflow}{{\psibm}}
\newcommand{\projnull}{{\Pbm}}

\newcommand{\methodabbrev}{NullFlow\xspace}

\newtheorem{proposition}{Proposition} %
\newtheorem{lemma}{Lemma} %

\makeatletter %
\newtheorem*{rep@proposition}{\rep@title}
    \newcommand{\newrepproposition}[2]{%
    \newenvironment{rep#1}[1]{%
     \def\rep@title{#2 \ref{##1}}%
     \begin{rep@proposition}}%
     {\end{rep@proposition}}}
\makeatother %

\newrepproposition{proposition}{Proposition}

\definecolor{lightgreen}{rgb}{.85,1,.85}
\definecolor{lightblue}{rgb}{.85,.85,1}
\definecolor{codefunccall}{RGB}{197, 5, 12} %
\definecolor{codemath}{RGB}{30, 80, 200}
\definecolor{codecommentteal}{RGB}{69, 123, 123}

\definecolor{DarkGreen}{rgb}{0.0, 0.5, 0.0} %

  {\endgroup}                    %

\begin{document}

\maketitle

\begin{abstract}
    We propose \methodabbrev, a principled framework for \emph{one-step} generative image reconstruction. 
    Our key idea is to confine the generative flow to a measurement-consistent subspace.
    Because the flow never leaves this subspace, \methodabbrev needs no separate data-fidelity corrections, unlike existing solvers. 
    \methodabbrev samples in a single network evaluation by learning the flow's \emph{average} velocity, avoiding the step-by-step integration of the traditional flow matching methods. 
    We prove that the average velocity of this constrained flow provides a consistent training objective for learning a one-step posterior sampler.
    We empirically demonstrate on image inpainting and accelerated MRI that \methodabbrev matches state-of-the-art diffusion solvers while cutting
    inference from hundreds of network evaluations to one.
    Code is available at 
    \href{https://github.com/uw-cig/NullFlow}{\textcolor{blue}{\text{https://github.com/uw-cig/NullFlow}}}.
\end{abstract}

\section{Introduction}

Imaging inverse problems seek to recover an unknown image
$\xbm \in \R^n$ from measurements $\ybm \in \R^m$.
For linear systems this is
modeled as
\begin{equation}
\label{eq:inverse_problem}
\ybm = \Abm \xbm + \nbm,
\end{equation}
where $\Abm \in \R^{m \times n}$ is a known forward operator and
$\nbm$ is noise. We focus on the ill-posed setting where
$\mathrm{null}(\Abm) \neq \{\zerobm\}$, which covers image inpainting,
accelerated MRI, and sparse-view CT. In this
setting, the measurements constrain only the row-space of
$\xbm$; the null-space component is invisible to $\Abm$ and must be supplied
\emph{entirely} by prior knowledge. Reconstruction is therefore not inversion
of $\Abm$, but generation of plausible null-space content consistent with the
measurements.

There is growing interest in learning-based methods for supplying this missing content. Early work predicted $\xbm$ directly from $\ybm$~\cite{dong.etal2016, zhang.etal2017, jin.etal2017, schwab.etal2019, chen.davies2020}. Recent
methods instead model the distribution of plausible images, using denoisers~\cite{Xu.etal2020, zhang.etal2021, renaud2024snore},
diffusion~\cite{ho2020denoising, rombach2022high} and
flow~\cite{lipman2022flow, liu2022flow, albergo2023building, albergo2025stochastic} models as priors that, plugged into
inverse solvers, reach state-of-the-art reconstruction across many
problems~\cite{kawar.etal2022, chung.etal2023, liu2023dolce, albergostochastic} (see~\cite{kamilov.etal2023, daras.etal2024} for surveys).

Existing generative solvers are expensive due to their iterative nature. They reconstruct images by simulating a full generative trajectory, alternating prior updates from the model with data-fidelity updates that enforce measurement consistency, and reach a solution only after hundreds or thousands of network evaluations~\cite{wang.etal2022, kim.etal2025, pourya.etal2025, martin.etal2025, chen2024flow, ben2024d, zhang2024flow, pokle2023training}. MeanFlows~\cite{geng.etal2025.mean, geng.etal2026.improvedmean, pixelmeanflows} are a recent class of methods
that remove this cost for unconditional and class-conditional generation. Rather than learning
the instantaneous velocity of a flow and integrating it step by step, a
MeanFlow learns the \emph{average} velocity between any two time points
directly, collapsing the entire trajectory into a single network
evaluation.
MeanFlows have recently been applied to image super-resolution, either through distillation~\cite{wang.etal2026, zhu.etal2026} or, in concurrent work, by training directly on a latent trajectory conditioned on the low-resolution input~\cite{niu.etal2026}. 
These methods operate in a learned latent space, where the forward operator has no tractable null space. 
NullFlow instead operates in pixel space, so its trajectory can be confined to the null space of the forward operator and measurement consistency holds by construction rather than being learned.

In this letter, we propose \methodabbrev, which brings the MeanFlow into the null space of the forward operator, thus enabling a one-step posterior sampler. We make three contributions:
\begin{itemize}
\item \emph{Method.} 
We propose a one-step solver for imaging inverse problems based on MeanFlows.
Confining the flow to $\mathrm{null}(\Abm)$
keeps every state measurement-consistent by construction, and learning its
average velocity yields reconstruction in a single network evaluation.
\item \emph{Theory.} We prove that the average
velocity of this constrained flow satisfies a MeanFlow identity restricted
to the null space, yielding a training objective whose global minimizer is
a one-step posterior sampler.
\item \emph{Experiments.} We show on image inpainting that one \methodabbrev sample gives the best LPIPS in a single step; averaging 100 samples approaches the MMSE estimate and exceeds the MSE-trained network in PSNR and SSIM, without retraining.
\end{itemize}

\begin{figure}[t]
    \centering
    \includegraphics[width=3.3in]{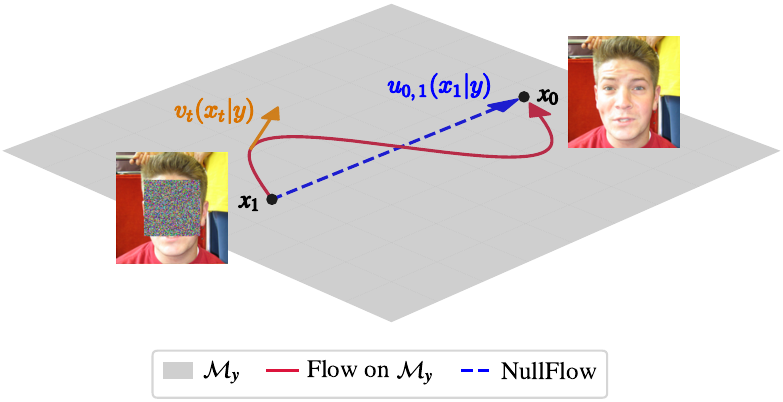}
    \caption{
        \methodabbrev is able to reconstruct in a single step without leaving the measurement-consistent subspace $\Mcal_\ybm := \{ \xbm : \Abm \xbm = \ybm \}$ by learning the average velocity \textcolor{blue}{$\meanvecfield_{0,1}(\xbm_1 | \ybm)$}
        rather than the instantaneous velocity \textcolor{orange}{$\vecfield_t(\xbm_t | \ybm)$} 
        of a \textcolor{red}{flow} confined to $\Mcal_\ybm$. 
    }
    \label{fig:flow illustration}
\end{figure}

\section{Proposed Method}

Our method is motivated by the observation that any $\xbm$ can be decomposed into the row space projection and null space projection of an ill-posed $\Abm$:
\[
    \xbm = \underbrace{\Abm^{\dagger} \Abm \xbm \vphantom{(\Ibm - \Abm^\dagger \Abm)}}_{\text{row space projection}} 
         + \underbrace{ (\Ibm - \Abm^\dagger \Abm) \xbm}_{\text{null space projection}},
\]
where $\Abm^{\dagger}$ is the pseudo-inverse of $\Abm$. We write
$\projnull := \Ibm - \Abm^\dagger \Abm$ for the orthogonal projection onto
$\mathrm{null}(\Abm)$, which we use throughout. In the noiseless measurement setting, the inverse problem can be viewed as reconstructing the null-space components, conditioned on the measurement $\ybm$;  the set of all images that are measurement consistent is the affine subspace
\begin{equation}
    \label{Eq:MeasurementSubspace}
    \Mcal_\ybm := \{ \xbm : \Abm \xbm = \ybm \}.
\end{equation}
Our objective is to learn a one-step MeanFlow, which we call \methodabbrev, 
that exists completely within $\Mcal_{\ybm}$
and samples from the conditional image distribution $\imageDist(\xbm | \ybm)$ (see
Fig.~\ref{fig:flow illustration}).

With measurement noise, $\ybm$ no longer determines the row-space component
exactly. One may instead anchor it at the least-squares estimate
$\Abm^\dagger\ybm$ and reconstruct only the null-space component 
by defining the flow within $\{ \xbm' : \Abm \xbm' = \Abm \Abm^\dagger \ybm \}$.
By anchoring to $\Abm^\dagger\ybm$, instead of seeking to sample from $\imageDist( \cdot | \ybm)$, such a flow attempts to sample the alternative target image distribution $\wbm \sim \imageDist(\cdot|\ybm)$ and forms
$\xbm' = \Abm^\dagger\ybm + \projnull\wbm$.
We note that this posterior distribution $\imageDist'(\xbm' | \ybm)$ is not the same as $\imageDist(\xbm | \ybm)$ and that such an anchoring is only effective at low noise.
For the remainder of the paper,
we assume noiseless measurements, 
and use $\Mcal_\ybm$ given in~\eqref{Eq:MeasurementSubspace}.

\begin{algorithm}[t]
\caption{\methodabbrev: Training}
\label{algo:training}
\footnotesize %
\begin{algorithmic}
    \REQUIRE clean data \texttt{x}, measurement \texttt{y}, forward model \texttt{\textcolor{codefunccall}{A}}, 
    \\\texttt{metric}(a,b)$=\|a-b\|_2^2$, \texttt{sample\_t\_r}: $t\sim\mathcal{U}[0.001,1)$, $r|t\sim\mathcal{U}[0,t)$
    \vspace{0.05cm} %

    \STATE \texttt{t,r = \textcolor{codefunccall}{sample\_t\_r}()}
    \STATE \COMMENT{Add noise only to manifold}
    \STATE \texttt{e = \textcolor{codefunccall}{randn}(x.shape)}
    \STATE \texttt{z = t\,\textcolor{codemath}{*}\,\textcolor{codefunccall}{proj\_null}(e)\,\textcolor{codemath}{+}\,(1\textcolor{codemath}{-}t)\,\textcolor{codemath}{*}\,\textcolor{codefunccall}{proj\_null}(x)\,\textcolor{codemath}{+}\,\textcolor{codefunccall}{A\_pinv}(y)}  %
    \vspace{0.5em} 
    \STATE \COMMENT{average velocity u from x-prediction}
    \STATE \texttt{\textcolor{codefunccall}{def} u\_fn(z, r, t)}:
    \STATE \hspace{2em} \texttt{\textcolor{codefunccall}{return} (z \textcolor{codemath}{-} net(z, r, t)) \textcolor{codemath}{/} t}
    \vspace{0.5em} 
    \STATE \COMMENT{instantaneous velocity v at time t}
    \STATE \texttt{v = u\_fn(z, t, t})
    \vspace{0.5em} 
    \STATE \COMMENT{predict u and dudt}
    \STATE \texttt{u, dudt = \textcolor{codefunccall}{jvp}(u\_fn, (z, r, t), (v, 0, 1))}
    \STATE \texttt{V = u \textcolor{codemath}{+} (t \textcolor{codemath}{-} r) \textcolor{codemath}{*} \textcolor{codefunccall}{stopgrad}(dudt)}
    \vspace{0.5em} 
    \STATE \texttt{loss = \textcolor{codefunccall}{metric}(V, e \textcolor{codemath}{-} x)}
\end{algorithmic}
\end{algorithm}

\begin{algorithm}[t]
\caption{\methodabbrev: One-Step Sampling}
\label{algo:sampling}
\footnotesize %
\begin{algorithmic}
\vspace{0.05cm} %

    \REQUIRE measurement \texttt{y}, forward model \texttt{\textcolor{codefunccall}{A}}, 
    \texttt{proj\_null}(x) = $(I-A^\dagger A)x$

    \STATE \texttt{e\_M = \textcolor{codefunccall}{A\_pinv}(y)
    \textcolor{codemath}{+}
    \textcolor{codefunccall}{proj\_null}(\textcolor{codefunccall}{randn}(x\_shape))}

    \STATE \texttt{z = net(e\_M, r=0, t=1)}

    \STATE \texttt{x\_pred = \textcolor{codefunccall}{A\_pinv}(y)
    \textcolor{codemath}{+}
    \textcolor{codefunccall}{proj\_null}(z)}

\end{algorithmic}
\end{algorithm}

\subsection{Subspace-Restricted Flow Matching}
\label{sec:flow matching}

We adapt the flow matching of~\cite{lipman2022flow} to the subspace
$\Mcal_\ybm$, confining every object---path, flow, and velocity---to
$\Mcal_\ybm$, whose tangent space is $\mathrm{null}(\Abm)$. This keeps every
state measurement-consistent without a separate data-fidelity step.

We transport between two distributions supported on $\Mcal_\ybm$. The
source is $\xbm_1 \sim \startDist(\cdot | \ybm)$, with
$\xbm_1 = \Abm^\dagger \ybm + \projnull\zbm$ and
$\zbm \sim \Ncal(\zerobm, \Ibm)$, which pins the row-space component to the
least-squares estimate $\Abm^\dagger \ybm$ and fills the null space with
Gaussian noise. The target is $\xbm_0 \sim \imageDist(\cdot | \ybm)$, the
conditional image distribution. 
We seek an interpolating probability path
$\imageDist_t(\cdot | \ybm)$ on $\Mcal_\ybm$, between
$\imageDist_1(\cdot|\ybm) = \startDist(\cdot|\ybm)$ and
$\imageDist_0(\cdot|\ybm) \approx \imageDist(\cdot|\ybm)$.

We take the conditional path toward a target $\xbm_0$ to be Gaussian with
covariance confined to $\mathrm{null}(\Abm)$,
\begin{equation}
    \label{eq: definition of p_t}
    \imageDist_t(\xbm | \xbm_0 , \ybm )
        = \mathcal{N}\big(\xbm \,\big|\, \mubm_t(\xbm_0 , \ybm),\, \sigma_t^2(\xbm_0) \projnull \big),
\end{equation}
with time-dependent mean $\mubm_t$ and scale $\sigma_t > 0$ subject to the
endpoint conditions $\mubm_1 = \Abm^\dagger\ybm,\, \sigma_1 = 1$ and
$\mubm_0 = \xbm_0,\, \sigma_0 = \sigma_{\min} \ll 1$, so that
$\imageDist_1 = \startDist(\cdot|\ybm)$ and $\imageDist_0$ concentrates at
$\xbm_0$. Since $\Abm \mubm_t = \ybm$ and $\sigma_t^2 \projnull$ spreads mass only along $\mathrm{null}(\Abm)$, the path
is supported on $\Mcal_\ybm$ for all $t$.
To realize this path deterministically, we use the location-scale map of the Gaussian, the conditional flow
\begin{equation}
    \label{eq: definition of condit flow}
    \conditflow_t(\xbm \,|\, \ybm) = \sigma_t(\xbm_0 ) \, \projnull \xbm + \mubm_t(\xbm_0, \ybm ),
\end{equation}
which scales the null-space content of $\xbm$ by $\sigma_t$ and centers it at
$\mubm_t$. 
Pushing a reference $\projnull\xbm \sim \Ncal(\zerobm, \projnull)$ through $\conditflow_t$ reproduces~\eqref{eq: definition of p_t} at every $t$.

The flow $\flow_t(\cdot | \ybm): \Mcal_{\ybm} \rightarrow \Mcal_{\ybm}$ is
generated by the time-dependent velocity field $\vecfield_t$, 
whose integral curves are its trajectories, solving
\begin{equation}
    \label{eq:flow ode}
    \frac{\d}{\d t} \flow_t(\xbm | \ybm) = \vecfield_t\big(\flow_t(\xbm | \ybm) \,\big|\, \ybm\big),
    \qquad \flow_1(\xbm | \ybm) = \xbm .
\end{equation}
The crucial constraint is that $\vecfield_t \in \mathrm{null}(\Abm)$: a step
along a null-space direction preserves $\Abm\xbm = \ybm$, so $\flow_t$ maps
$\Mcal_\ybm$ into itself and every state remains measurement-consistent.

We instantiate the mean and scale to vary linearly in $t$,
\begin{equation}
    \label{eq:schedules}
    \mubm_t(\xbm_0 , \ybm) = \big(1 - (1-\sigma_{\text{min}})t \big) \, \projnull \xbm_0 + \Abm^\dagger \ybm,
    \quad
    \sigma_t(\xbm_0) = t,
\end{equation}
so that~\eqref{eq: definition of condit flow} becomes
\begin{equation}
    \label{eq: the flow we use}
    \conditflow_t(\xbm \,|\, \ybm) = t\, \projnull \xbm +  \big(1 - (1-\sigma_{\text{min}})t\big)\, \projnull \xbm_0 + \Abm^\dagger \ybm,
\end{equation}
moving the null-space content from $\projnull\xbm$ to $\projnull\xbm_0$ while
holding the anchor $\Abm^\dagger\ybm$ fixed.
Differentiating~\eqref{eq: the flow we use} in $t$ and re-expressing at the
current state $\xbm$ yields the velocity field
\begin{equation}
    \label{eq:instantaneous velocity}
    \vecfield_t(\xbm | \ybm) = \projnull\big(\xbm - (1 - \sigma_{\text{min}})\,\xbm_0 \big) \in \mathrm{null}(\Abm) .
\end{equation}
Lemmas~\ref{lemma: my Lipman thm 1} and~\ref{lemma: my Lipman thm 3} in the appendix extend
\cite{lipman2022flow} to show $\vecfield_t$
generates~\eqref{eq: definition of p_t} on $\Mcal_\ybm$, so the flow transports
$\startDist(\cdot|\ybm)$ to $\imageDist(\cdot|\ybm)$.

\begin{figure}[t]
    \centering
    \includegraphics[width=\linewidth]{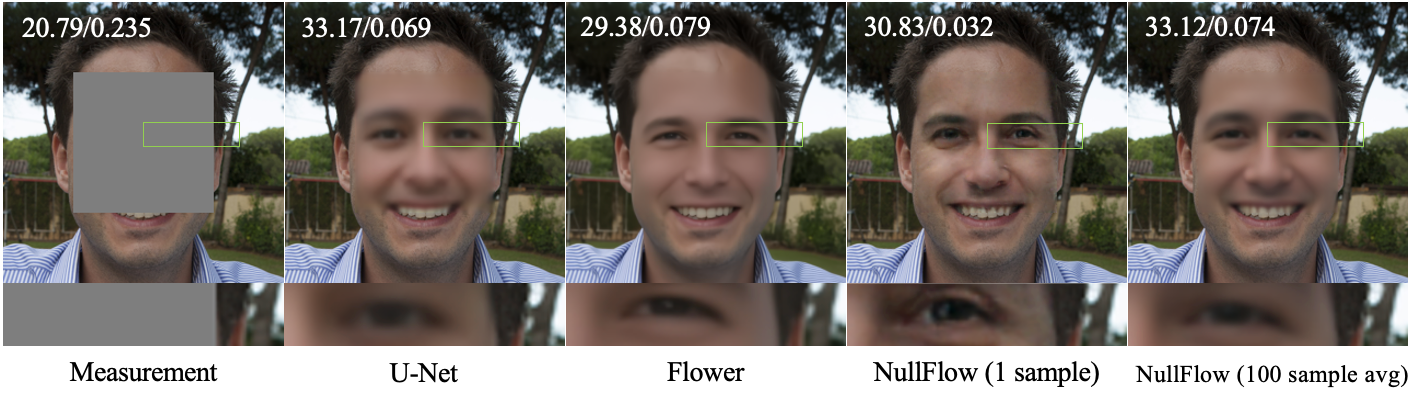}
    \caption{ Reconstructions on a test image. Left to right: measurement, U-Net, Flower, a single NullFlow sample, and the average of 100 NullFlow samples. PSNR/LPIPS are shown in the top-left corner of each. A single NullFlow sample is sharp and perceptually faithful (best LPIPS), whereas averaging many samples approaches the MMSE estimate, trading perceptual quality for lower distortion and visibly resembling the MSE-trained U-Net. }
    \label{fig: comparison}
\end{figure}
\subsection{One-Step Sampling via MeanFlow}
\label{sec:meanflow}

Sampling with the flow of Section~\ref{sec:flow matching} requires integrating
the instantaneous velocity~\eqref{eq:instantaneous velocity} from $t=1$ to
$t=0$, which costs many network evaluations. To reconstruct in a single step,
we instead learn the \emph{average} velocity of the flow, following the Mean
Flow framework~\cite{geng.etal2025.mean}, but restricted to $\Mcal_\ybm$.

Writing $\meanvecfield$ for the average velocity and $\vecfield$ for the
instantaneous velocity, the \emph{average conditional velocity} between two
times $r \leq t$ is the time-average of $\vecfield$ along the trajectory,
\begin{equation}
    \label{eq: average conditional velocity}
    \meanvecfield_{r, t}(\xbm_t | \ybm) := \frac{1}{t-r} \int_r^t \vbm_{\tau}(\xbm_{\tau} | \ybm) \, \d\tau,
\end{equation}
with $0 \leq r \leq t \leq 1$. The value of learning $\meanvecfield_{r,t}$ directly is that it enables production of a sample by a single jump from $r=1$ to $t=0$, with no step-by-step integration of $\vecfield$.

We convert~\eqref{eq: average conditional velocity} into a pointwise constraint
relating $\meanvecfield$ and $\vecfield$ that is better suited for learning.
Multiplying by $(t-r)$ and differentiating in $t$ gives the \emph{MeanFlow
identity} over $\Mcal_\ybm$,
\begin{equation}
    \label{eq: meanflow identity}
    \meanvecfield_{r, t}(\xbm_t | \ybm) = \vbm_t(\xbm_t | \ybm) - (t-r) \frac{\d}{\d t} \meanvecfield_{r, t}(\xbm_t | \ybm),
\end{equation}
which reduces to
$\meanvecfield_{t,t} = \vbm_t$ when $r = t$. Crucially,
\eqref{eq: meanflow identity} involves only $\vecfield_t$ and a derivative of
$\meanvecfield_{r,t}$, both available without integration.
This identity yields a training objective whose minimizer is the desired
average velocity. We parametrize $\meanvecfield^{\thetabm}$ by a network and
regress it against the right-hand side of~\eqref{eq: meanflow identity}, with
$\vecfield_t$ supplied in closed form by~\eqref{eq:instantaneous velocity}.

\begin{proposition} \label{prop:main theorem}
    Let $\Abm$ be a fixed matrix, $\ybm$ the noiseless measurement, and
    $\projnull := \Ibm - \Abm^\dagger \Abm$ the null-space projection.
    If $\meanvecfield^{\thetabm^*}_{r, t}$ satisfies the conditional MeanFlow Identity
        \begin{equation*}
            \meanvecfield^{\thetabm^*}_{r, t}(\xbm_t | \ybm) = \vbm_t(\xbm_t | \ybm) - (t-r) \frac{\d}{\d t} \meanvecfield^{\thetabm^*}_{r, t}(\xbm_t | \ybm),
        \end{equation*}
    then $\meanvecfield^{\thetabm^*}_{r, t}(\xbm | \ybm )$ is a MeanFlow in the null space of $\Abm$ such that $\meanvecfield_{0, 1}(\xbm_t | \ybm)$  samples from $\imageDist_0(\xbm | \ybm) \approx \imageDist(\xbm | \ybm)$.
    In particular, achieving
    \begin{equation}
        \label{eq:prop objective}
        \E\left\| \meanvecfield^{\thetabm^*}_{r,t}(\zbm | \ybm) - \vecfield + (t-r) \frac{\d}{\d t}\meanvecfield^{\thetabm}_{r,t}(\zbm | \ybm) \right\|_2^2 = 0,
    \end{equation}
    with expectation over $r \leq t \sim \Ucal[0,1]$, $(\xbm_1, \ybm) \sim \startDist(\xbm_1, \ybm)$, $\bm{\epsilon} \sim \Ncal(\zerobm, \Ibm)$,
    and
    \begin{align*}
        \zbm &:= (1 - (1-\sigma_{\min})t)\, \projnull \bm{\epsilon} + t\, \projnull \xbm_1 + \Abm^\dagger \ybm, \\
        \vecfield &:= \projnull(\xbm_1 - (1 - \sigma_{\min}) \bm{\epsilon} ),
    \end{align*}
    is sufficient to satisfying the conditional MeanFlow Identity.
\end{proposition}
The proof is given in the Appendix. 
This result extends
the unconditional MeanFlow of~\cite{geng.etal2025.mean}
by establishing that our subspace-restricted flow transports $\startDist(\cdot|\ybm)$ to $\imageDist(\cdot|\ybm)$ on $\Mcal_\ybm$ and that reaching zero loss enforces the MeanFlow identity~\eqref{eq: meanflow identity}.
Algorithms~\ref{algo:training} and~\ref{algo:sampling} provide the implementation details of \methodabbrev, where reconstruction is initialized with $\Abm^\dagger\ybm + \projnull\bm{\epsilon}$ and
$\meanvecfield^{\thetabm^*}_{0,1}$ is evaluated once.

\section{Numerical Evaluation}
\label{sec:numerical evaluation}

We validate \methodabbrev by considering 
block and random inpainting on natural images and Cartesian undersampled MRI at acceleration rates $R=4$ and $R=8$. During training, we add white Gaussian noise only to the unobserved measurement components, while keeping the observed measurements unchanged to ensure measurement consistency.
All NullFlow models use DDPM-style U-Nets conditioned on the time pair $(r,t)$ and $\ybm$ and are trained on one NVIDIA RTX PRO 6000 GPUs. Additional details will be made available in our GitHub repository.

We follow Pixel MeanFlow~\cite{pixelmeanflows} to minimize the MeanFlow objective.
With the direct output of the network $\xbm^{\theta}_{r,t}(\zbm_t)$, we compute the average velocity with
    \(\ubm^\theta_{r,t}(\zbm_t | \ybm) = (1/t)(\zbm_t - \xbm^{\theta}_{r,t}(\zbm_t)).\)
We train the network with the PMF loss
    \begin{equation}
        \mathcal{L}_{\mathrm{PMF}}
        :=
        \mathop{\mathbb{E}}\limits_{r,t,\xbm,\ybm,\bm{\epsilon}}
        \left\|
            \Vbm^{\theta}_{r,t}(\xbm_t|\ybm)
            -
            (\epsilonbm - \xbm)
        \right\|_2^2  \label{eq:lpmf} \\
    \end{equation}
    where
    \begin{equation*}
        \Vbm^{\theta}_{r,t}(\xbm_t|\ybm) := \ubm^{\theta}_{r,t}(\xbm | \ybm) + (t-r) \cdot \texttt{JVP}_{\texttt{sg}}(\ubm^{\theta}_{r,t}, \epsilonbm - \xbm). \nonumber 
    \end{equation*}
    $\texttt{JVP}_{\texttt{sg}}$ denotes the stop-gradient on the $\frac{\texttt{d}}{\texttt{d}t}\ubm^{\theta}_{r,t}$ outcome.
    We refer \cite{pixelmeanflows} for more detail of PMF objective.
    At inference, we sample in one step by following Algorithm~\ref{algo:sampling}.
We compare against measurement-conditioned baselines (MSE regressor and CMF), which use measurements during training, and pretrained-prior methods, which incorporate measurements only during reconstruction. CMF differs from NullFlow only in its Gaussian source.

As shown in Table~\ref{table:comparison_table1}, \methodabbrev achieves the best LPIPS on box inpainting and the second-best on random inpainting, using only one function evaluation. It also achieves the second-highest PSNR on both MRI tasks, demonstrating competitive reconstruction quality with minimal inference cost. 
The last column of Table~\ref{table:comparison_table1} 
illustrates the perception--distortion tradeoff: \methodabbrev produces perceptually better single-sample reconstructions, while sample averaging improves PSNR and SSIM at the cost of LPIPS, approaching an MMSE-like estimate.

\begin{table*}[t]
    \centering
    \small
    \caption{Quantitative results on image inpainting and MRI reconstruction.
    \hlgreen{Best values} and \hlblue{second-best} are color coded per metric.}
    \label{table:comparison_table1}

    \vspace{0.05cm}
    \renewcommand{\arraystretch}{0.85}
    \setlength{\tabcolsep}{3pt}

    \resizebox{\textwidth}{!}{%
    \begin{tabular}{lllcccccccccc}
    \textbf{Task} & \textbf{Setting} & \textbf{Metric}
    & \textbf{Input}
    & \textbf{MSE regressor}
    & \textbf{CMF}
    & \textbf{DPIR}
    & \textbf{DPS}
    & \textbf{DiffPIR}
    & \textbf{PnP-Flow}
    & \textbf{FLOWER}
    & \textbf{NullFlow}
    & \textbf{NullFlow (100 sample avg)} \\
    \midrule

    & & NFE$\downarrow$
    & -- & \hlgreen{1} & \hlgreen{1}
    & 200 & 1000 & 200
    & 500 & 500 & \hlgreen{1} & \hlblue{100} \\

    \midrule

    \multirow[c]{7}{*}{\raisebox{-3.5ex}{\textbf{Inpainting}}}
    & & Time / sample (s)$\downarrow$
    & -- & \hlgreen{0.011} & \hlgreen{0.011}
    & 1.00 & 23.71 & 2.42
    & 12.45 & 12.61 & \hlgreen{0.011} & \hlblue{0.981} \\

    \cmidrule(lr){2-13}

    & \multirow[c]{3}{*}{\textbf{Box}}
    & LPIPS$\downarrow$
    & 0.272 & 0.089 & 0.097 & 0.231 & 0.157
    & \hlblue{0.077} & 0.089 & 0.105
    & \hlgreen{0.070} & 0.086 \\

    & & PSNR$\uparrow$
    & 18.70 & \hlblue{27.11} & \hlgreen{27.52}
    & 19.43 & 23.89 & 25.44 & 24.02 & 24.30
    & 24.42 & 26.72 \\

    & & SSIM$\uparrow$
    & 0.842 & \hlblue{0.911} & \hlgreen{0.919}
    & 0.848 & 0.829 & 0.905 & 0.893 & 0.897
    & 0.887 & 0.908 \\

    \cmidrule(lr){2-13}

    & \multirow[c]{3}{*}{\textbf{Random}}
    & LPIPS$\downarrow$
    & 1.075 & 0.034 & 0.035 & 0.042 & 0.067
    & \hlgreen{0.031} & 0.040 & 0.041
    & \hlblue{0.033} & 0.075 \\

    & & PSNR$\uparrow$
    & 12.83 & \hlgreen{35.14} & 33.83
    & \hlblue{34.01} & 30.98 & 33.49
    & 33.79 & 33.77 & 31.14 & 32.16 \\

    & & SSIM$\uparrow$
    & 0.215 & \hlgreen{0.957} & 0.930
    & \hlblue{0.944} & 0.884 & 0.927
    & 0.937 & 0.936 & 0.897 & 0.911 \\

    \midrule

    \multirow[c]{7}{*}{\raisebox{-3.5ex}{\textbf{CS-MRI}}}
    & & Time / sample (s)$\downarrow$
    & -- & \hlgreen{0.08} & \hlgreen{0.08}
    & \hlblue{0.40} & 26.93 & 14.12
    & 5.14 & 7.48 & \hlgreen{0.08} & 9.80 \\

    \cmidrule(lr){2-13}

    & \multirow[c]{3}{*}{$\mathbf{4\times}$}
    & LPIPS$\downarrow$
    & 0.241 & 0.070 & 0.065 & 0.053 & 0.081
    & 0.065 & 0.190 & \hlgreen{0.023}
    & \hlblue{0.049} & 0.054 \\

    & & PSNR$\uparrow$
    & 25.38 & \hlgreen{36.80} & 32.51 & 32.65 & 30.22
    & 30.44 & 30.40 & \hlblue{35.16}
    & 33.54 & 34.09 \\

    & & SSIM$\uparrow$
    & 0.793 & \hlblue{0.931} & 0.905 & 0.876
    & 0.676 & 0.698 & 0.839
    & \hlgreen{0.949} & 0.903 & 0.921 \\

    \cmidrule(lr){2-13}

    & \multirow[c]{3}{*}{$\mathbf{8\times}$}
    & LPIPS$\downarrow$
    & 0.363 & 0.117 & 0.120 & 0.141 & 0.117
    & \hlblue{0.101} & 0.239 & \hlgreen{0.063}
    & 0.106 & 0.108 \\

    & & PSNR$\uparrow$
    & 22.00 & \hlgreen{33.88} & 28.22 & 26.68 & 25.95
    & 26.62 & 25.58 & 29.21
    & 29.95 & \hlblue{30.33} \\

    & & SSIM$\uparrow$
    & 0.693 & \hlgreen{0.892} & 0.824 & 0.775 & 0.610
    & 0.637 & 0.761 & \hlblue{0.886}
    & 0.855 & 0.865 \\

    \end{tabular}%
    }
\end{table*}

\section{Conclusion}
We introduced \methodabbrev, a one-step generative solver for imaging
inverse problems. By confining a MeanFlow to the null space of the forward
operator, every state stays measurement-consistent by construction and
\methodabbrev samples from 
$\imageDist_0(\xbm|\ybm) \approx \imageDist(\xbm|\ybm)$ in a single network
evaluation. On image inpainting it matches the perceptual quality of iterative
diffusion and flow solvers at a fraction of the cost. Extending \methodabbrev
to noisy measurements and to sample-varying operators are promising future
directions.

\appendix
\section*{Appendix}

Lemmas~\ref{lemma: my Lipman thm 1},~\ref{lemma: my Lipman thm 2}, and~\ref{lemma: my Lipman thm 3} extend
Theorems 1,2, and 3, respectively, of~\cite{lipman2022flow} to measurement-conditional probability paths and vector fields on $\Mcal_\ybm$. 
Building on this affine-subspace flow,
Prop.~\ref{prop:main theorem} extends~\cite{geng.etal2025.mean} to a MeanFlow restricted to $\Mcal_\ybm$.

\section{Technical Lemmas} \label{app:lemmas}
\begin{lemma} \label{lemma: my Lipman thm 1}
    Let $\vecfield_t(\xbm | \xbm_0, \ybm)$ be conditional vector fields on $\Mcal_\ybm$ that generate probability path  $p_t(\xbm | \xbm_0, \ybm)$ on $\Mcal_\ybm$.
    For any distribution $q(\xbm_0 | \ybm)$ on $\Mcal_\ybm$, the marginal vector field
    \begin{equation*}
        \vecfield_t(\xbm | \ybm) = \int_{\Mcal_\ybm} \vecfield_t(\xbm | \xbm_0, \ybm) \frac{p_t(\xbm | \xbm_0, \ybm) q(\xbm_0 | \ybm)}{p_t(\xbm | \ybm)} \, \d \mathrm{vol}_{\Mcal_\ybm}
    \end{equation*}
    generates the marginal probability path
    \begin{equation*}
        p_t(\xbm | \ybm) = \int_{\Mcal_\ybm} p_t(\xbm | \xbm_0, \ybm) q(\xbm_0 | \ybm) \, \d \mathrm{vol}_{\Mcal_\ybm}.
    \end{equation*}
\end{lemma}
\begin{proof}
    The proof follows that of Theorem 1 in~\cite{lipman2022flow}, with the modification of taking integrals over the manifold, $\int d\mathrm{vol}_{\Mcal_{\ybm}}$, rather than all of $\R^n$, ($\int \d\xbm$).
\end{proof}
To train Flow Matching (FM) models, the following intractable loss must be minimized:
\begin{equation*}
    \Lcal_{FM}(\theta) = \E_{t, p_t(\xbm)} \| \vecfield_t^{\theta}(\xbm | \ybm ) - \vecfield_t(\xbm | \ybm ) \|^2.
\end{equation*}
Lemma 2 shows that Conditional Flow Matching (CFM) loss
\begin{equation*}
    \Lcal_{CFM}(\theta) = \E_{t, \imageDist(\xbm_1), p_t(\xbm | \xbm_1)} \| \vecfield_t^{\theta}(\xbm | \ybm ) - \vecfield_t(\xbm | \ybm, \xbm_1 ) \|^2,
\end{equation*}
can be used instead of the Flow Matching (FM) loss. 
For more details regarding FM and CFM, we refer the reader to~\cite{lipman2022flow}.
\begin{lemma} \label{lemma: my Lipman thm 2}
    Assuming that $p_t(\xbm | \ybm) > 0$ for all $\xbm \in \Mcal_\ybm$ and $t \in [0,1]$, then $\nabla_\theta \Lcal_{FM}(\thetabm) = \nabla_\theta \Lcal_{CFM}(\thetabm)$.
\end{lemma}
\begin{proof}
    The proof follows that of Theorem 2 in~\cite{lipman2022flow}, with the modification of taking integrals over the manifold, $\int \d\mathrm{vol}_{\Mcal_{\ybm}}$, rather than all of $\R^n$, ($\int \d\xbm$).
\end{proof}

\begin{lemma} \label{lemma: my Lipman thm 3}
Let $p_t(\xbm | \xbm_0, \ybm)$ be a Gaussian probability path restricted to $\Mcal_\ybm$ as defined in Eq.~\eqref{eq: definition of p_t}, and $\conditflow_t$ its corresponding flow map as defined in Eq.~\eqref{eq: definition of condit flow}. 
Then, the unique vector field that defines $\conditflow_t$ has the form
\begin{equation*}
    \vecfield_t(\xbm|\xbm_0, \ybm)  = \frac{\dot{\sigma}_t(\xbm_0 )}{ \sigma_t(\xbm_0 )} \big(\xbm - \mubm_t(\xbm_0 , \ybm) \big) + \dot{\mubm}_t(\xbm_0 , \ybm),
\end{equation*}
for $\xbm, \xbm_1 \in \Mcal_\ybm$.
Consequently, $\vecfield_t(\xbm|\xbm_0, \ybm)$ generates the Gaussian path $p_t(\xbm | \xbm_0, \ybm)$.
\end{lemma}
\begin{proof}
We need to satisfy the conditional probability path $\dot{\conditflow_t}(\xbm | \xbm_0, \ybm) = \vecfield_t( \conditflow_t(\xbm | \xbm_0, \ybm) | \xbm_0, \ybm)$.
Let $\xbm \in \Mcal_\ybm$ and $\zbm = \conditflow_t(\xbm | \xbm_0, \ybm)$.
By definition of the flow, $\zbm \in \Mcal_\ybm$.
First, notice that $\conditflow_t$ is a bijection when restricted to $\Mcal_\ybm$,
and so we can write $\xbm = \conditflow^{-1}_t(\zbm | \xbm_0, \ybm)$.
Therefore,
\begin{equation} \label{eq:vf proof step 1}
    \vecfield_t(\zbm | \xbm_0, \ybm) = \dot{\conditflow}_t \big(\conditflow^{-1}_t(\zbm | \xbm_0, \ybm) | \xbm_0, \ybm \big).
\end{equation}
Since $\sigma_t > 0$,
the inverse on $\Mcal_\ybm$ is given by
\begin{align}
    \conditflow_t^{-1}(\zbm | \xbm_0, \ybm) &= \sigma^{-1}_t(\xbm_0 ) \projnull \big(\zbm - \mubm_t(\xbm_0 , \ybm) \big) + \Abm^\dagger\ybm \nonumber \\
         &= \sigma^{-1}_t(\xbm_0 ) \big(\zbm - \mubm_t(\xbm_0 , \ybm) \big) + \Abm^\dagger\ybm,   \label{eq:vf proof step 2}
\end{align}
where we use that $\zbm - \mubm_t  \in \mathrm{null}(\Abm)$ since $\zbm, \mubm_t \in \Mcal_\ybm$.
We add $\Abm^\dagger \ybm$ to recover $\zbm \in \Mcal_\ybm$.  
The time derivative of $\conditflow_t$ is
\begin{align} 
    \dot{\conditflow}_t(\xbm | \xbm_0, \ybm) &= \dot{\sigma}_t(\xbm_0) \projnull \xbm + \dot{\mubm}_t(\xbm_0 , \ybm) \label{eq:vf proof step 3}
\end{align}
Plugging Eq.~\eqref{eq:vf proof step 1} and Eq.~\eqref{eq:vf proof step 2} into Eq.~\eqref{eq:vf proof step 3} gives:
\begin{align*}
    \vecfield_t&(\zbm | \xbm_0 , \ybm) = \dot{\conditflow}_t \Big( \sigma^{-1}_t(\xbm_0 ) \big(\zbm - \mubm_t(\xbm_0 , \ybm) \big) + \Abm^\dagger\ybm  \mid \xbm_0, \ybm \Big) \\
        &= \frac{\dot{\sigma}_t(\xbm_0 )}{\sigma_t(\xbm_0 )} \projnull\big(\zbm - \mubm_t(\xbm_0 , \ybm) \big) + \dot{\sigma}_t(\xbm_0 ) \projnull\Abm^\dagger\ybm  + \dot{\mubm}_t(\xbm_0 , \ybm) \\
        &= \frac{\dot{\sigma}_t(\xbm_0)}{\sigma_t(\xbm_0 )} \big(\zbm - \mubm_t(\xbm_0 , \ybm) \big)  + \dot{\mubm}_t(\xbm_0 , \ybm),
\end{align*}
where we use $\zbm - \mubm_t \in \mathrm{null}(\Abm)$ and $\Abm^\dagger \ybm \not \in \mathrm{null}(\Abm)$.
\end{proof}

\section{Proof of Proposition~\ref{prop:main theorem}}
Lemmas~\ref{lemma: my Lipman thm 1}--\ref{lemma: my Lipman thm 3} verify that we can learn the instantaneous velocity of a flow restricted to $\Mcal_\ybm$.
Lemma~\ref{lemma: my Lipman thm 3} establishes that a vector field $\vecfield_t(\xbm|\xbm_0, \ybm)$ can be defined to generate the Gaussian path from $p_1(\xbm | \xbm_0, \ybm)$ to $p_0(\xbm | \xbm_0, \ybm)$.
Lemma~\ref{lemma: my Lipman thm 1} establishes that from $\vecfield_t(\xbm|\xbm_0, \ybm)$, a marginal vector field $\vecfield_t(\xbm|\ybm)$ can be defined to generate the marginal probability path $p_t(\xbm | \ybm)$.
Finally, Lemma~\ref{lemma: my Lipman thm 2} establishes that such a marginal instantaneous vector field can be feasibly trained.

We now seek the average conditional velocity: 
using the same arguments from~\cite{geng.etal2025.mean}, 
 we rewrite the \textit{MeanFlow Identity};
differentiating~\eqref{eq: average conditional velocity} gives~\eqref{eq: meanflow identity}, 
where we use the Fundamental Theorem of Calculus.
To see sufficiency of the MeanFlow Identity, we refer the reader to section B.3 of~\cite{geng.etal2025.mean}.
Rearranging terms gives the measurement conditional MeanFlow Identity~\eqref{eq: meanflow identity}.
Finally, we use the defined conditional flow in Eq.~\eqref{eq: the flow we use}, which gives the instantaneous vector field $\vecfield_t$,
\begin{align*} 
    \vecfield_t(\xbm_t | \ybm) = \frac{\d}{\d t}\conditflow_t(\xbm | \ybm) &= -(1 - \sigma_{\text{min}}) \projnull \xbm + \projnull \xbm_1 \\ & = \projnull (\xbm_1 - (1 - \sigma_{\text{min}})  \xbm ) .
\end{align*}
If a zero loss in Eq.~\eqref{eq:prop objective} is achieved, then the MeanFlow Identity Eq.~\eqref{eq: meanflow identity} is satisfied. \qed

{
\small

\bibliographystyle{IEEEbib}
\bibliography{refs}
}

\end{document}